\documentclass[10pt,journal,compsoc]{IEEEtran}
\usepackage{times}
\usepackage{epsfig}
\usepackage{graphicx}
\usepackage{amsmath}
\usepackage{amssymb}
\usepackage{amsthm}
\usepackage{amsfonts,mathtools}
\usepackage{etoolbox}
\usepackage{xspace}
\usepackage{xcolor}
\usepackage{array,multirow}
\usepackage{boldline}
\usepackage[titlenumbered,ruled]{algorithm2e}
\usepackage[symbol]{footmisc}
\usepackage{tablefootnote}
\usepackage[flushleft]{threeparttable}
\usepackage{url}
\usepackage{bm}

\newtheorem{theorem}{Theorem}

\newtheorem{prop}{Proposition}

\newcommand\scalemath[2]{\scalebox{#1}{\mbox{\ensuremath{\displaystyle #2}}}}
\newcommand{\MyMapTemplatePrefixc}[4]{\expandafter#1\csname#3#4\endcsname{#2{#4}}} % it remembles a template: \#3#4 --> #2{#4}
\forcsvlist{\MyMapTemplatePrefixc {\def} {\mathcal}{c}} {A,B,C,D,E,F,G,H,I,J,K,L,M,N,O,P,Q,R,S,T,U,V,W,X,Y,Z}  % e.g., \cA --> \mathcal{A}

\newcommand{\MyMapTemplatePrefixtb}[5]{\expandafter#1\csname#4#5\endcsname{#2{#3{#5}}}} % it remembles a template: \#3#4 --> #2{#4}
\forcsvlist{\MyMapTemplatePrefixtb {\def} {\tilde}{\mathbf}{t}} {A,B,C,D,E,F,G,H,I,J,K,L,M,N,O,P,Q,R,S,T,U,V,W,X,Y,Z}  % e.g., \tA --> \tilde{A}
\forcsvlist{\MyMapTemplatePrefixtb {\def} {\tilde}{\mathbf}{t}} {0,1,a,b,c,d,e,f,g,h,i,j,k,l,m,n,o,p,q,r,s,u,v,w,x,y,z}  % e.g., \ta --> \tilde{a}

\newcommand{\MyMapTemplateNoPrefix}[3]{\expandafter#1\csname#3\endcsname{#2{#3}}}
\forcsvlist{\MyMapTemplateNoPrefix {\def} {\mathbf} } {0,1,a,b,c,d,e, f, g, h, i, j, k, l, m, n, o, p, q, r, u, v, w, x, y, z} % e.g., \a --> \mathbf{a}
\forcsvlist{\MyMapTemplateNoPrefix {\def} {\mathbf} } {A,B,C,D,E,F,G,H,I,J,K,L,M,N,O,P,Q,R,S,T,U,V,W,X,Y,Z}  % e.g., \A --> \mathcal{A}

\newcommand\paragraf[1]{\noindent\textbf{#1}}

\def\resp{\emph{resp.}\@\xspace}
\def\etc{\emph{etc.}\@\xspace}
\def\eg{\emph{e.g.}\@\xspace}
\def\ie{\emph{i.e.}\@\xspace}
\def\etal{\emph{et~al.}\@\xspace}

\newcommand{\newcontent}[1]{\textcolor{blue}{#1}}

\usepackage[pagebackref=true,breaklinks=true,colorlinks,bookmarks=false]{hyperref}
\usepackage[noadjust]{cite}

\begin{document}

\begin{figure}
	\centering
	{\includegraphics[width=1\textwidth]{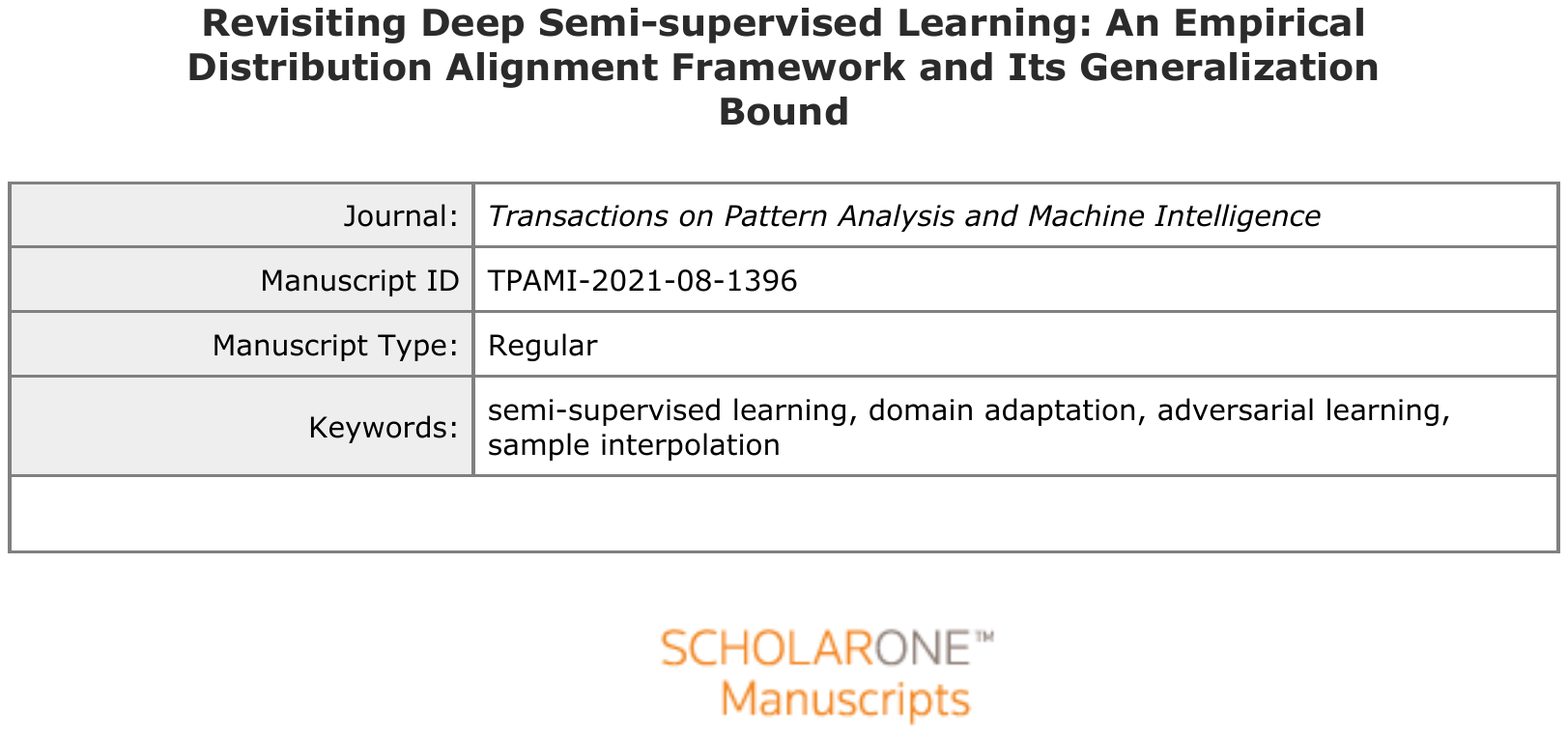}}
\end{figure}

\title{Revisiting Deep Semi-supervised Learning: \\ An Empirical Distribution Alignment Framework and Its Generalization Bound}
\author{Feiyu~Wang,~\IEEEmembership{}
        Qin~Wang,~\IEEEmembership{}
        Wen~Li,~\IEEEmembership{}
        Dong~Xu,~\IEEEmembership{Fellow,~IEEE}
        and~Luc~Van~Gool~\IEEEmembership{}% <-this % stops a space
\IEEEcompsocitemizethanks{
\IEEEcompsocthanksitem Feiyu Wang and Dong Xu are with the School of Electrical and Information Engineering, The University of Sydney, NSW, Australia.\protect\\ E-mail: fwan6592@uni.sydney.edu.au, dong.xu@sydney.edu.au.

\IEEEcompsocthanksitem Qin Wang and Luc Van Gool are with the Department of Information Technology and Electrical Engineering, Swiss Federal Institute of Technology in Z\"urich, Z\"urich, Switzerland.\protect\\ E-mail: qwang@ethz.ch, vangool@vision.ee.ethz.ch.
\IEEEcompsocthanksitem Wen Li is with the School of Computer Science and Engineering, University of Electronic Science and Technology of China, Chengdu, China.\protect\\ E-mail: liwenbnu@gmail.com.
}% <-this % stops an unwanted space
\thanks{Manuscript received XXXX XX, XXXX; revised XXXX XX, XXXX.}}

% The paper headers
\markboth{IEEE XXXX ON XXXX~XXXX,~Vol.~XX, No.~XX, XXXX~XXXX}%
{Author~1 \MakeLowercase{\textit{et al.}}: Revisiting Deep Semi-supervised Learning: A New Framework and Its Generalization Bound}
\IEEEtitleabstractindextext{%
\begin{abstract}
Deep semi-supervised learning (SSL), which aims to utilize a limited number of labeled samples and abundant unlabeled samples to learn robust deep models, is attracting increasing interest from both machine learning and computer vision researchers. 
In this work, we revisit the semi-supervised learning problem from a new perspective of explicitly reducing empirical distribution mismatch between labeled and unlabeled samples. Benefited from this new perspective, we first propose a new deep semi-supervised learning framework called Semi-supervised Learning by Empirical Distribution Alignment (SLEDA), in which existing technologies from the domain adaptation community can be readily used to address the semi-supervised learning problem through reducing the empirical distribution distance between labeled and unlabeled data. Based on this framework, we also develop a new theoretical generalization bound for the research community to better understand the semi-supervised learning problem, in which we show the generalization error of semi-supervised learning can be effectively bounded by minimizing the training error on labeled data and the empirical distribution distance between labeled and unlabeled data. Building upon our new framework and the theoretical bound, we develop a simple and effective deep semi-supervised learning method called Augmented Distribution Alignment Network (ADA-Net) by simultaneously adopting the well-established adversarial training strategy from the domain adaptation community and a simple sample interpolation strategy for data augmentation. Additionally, we incorporate both strategies in our ADA-Net into two exiting SSL methods to further improve their generalization capability, which indicates that our new framework provides a complementary solution for solving the SSL problem. Our comprehensive experimental results on two benchmark datasets SVHN and CIFAR-10 for the semi-supervised image recognition task and another two benchmark datasets ModelNet40 and ShapeNet55 for the semi-supervised point cloud recognition task demonstrate the effectiveness of our proposed framework for SSL.
\end{abstract}

}

\maketitle

\IEEEdisplaynontitleabstractindextext

\IEEEpeerreviewmaketitle

\IEEEraisesectionheading{\section{Introduction}\label{sec:introduction}}
\IEEEPARstart{S}{emi-supervised} Learning (SSL) methods aim to learn robust models with a limited number of labeled samples and an abundant number of unlabeled samples. As a classical learning paradigm, it has gained much interest from both machine learning and computer vision community. Many approaches have been proposed in recent decades, including label propagation, graph regularization, \etc~\cite{chapelle2003cluster, blum1998combining, belkin2004semi, grandvalet2005semi, blum2001learning, zhu2005semi}. Recently, there is increasing research interest in training deep neural networks by using the semi-supervised learning methods \cite{tarvainen2017mean, laine2017temporal, miyato2018virtual, odena2018realistic, cicek2018saas, Chen_2018_ECCV}. This is partially due to the requirement of a large amount of labeled training data for training robust deep models, which unfortunately involves heavy data annotation costs.

\begin{figure}
	\centering
	{\includegraphics[width=\columnwidth]{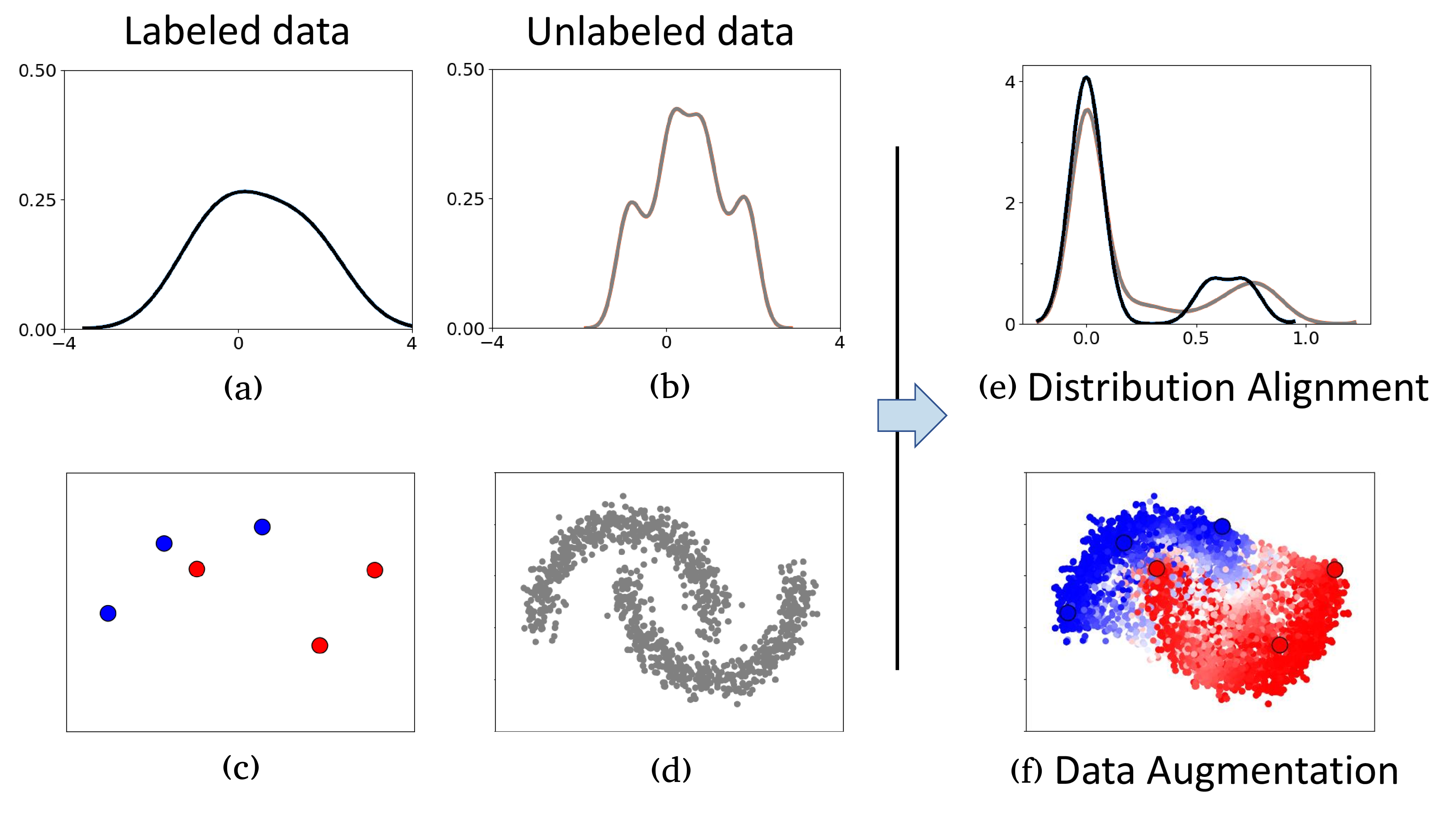}}
	\caption{Illustration of empirical distribution mismatch between the labeled and unlabeled samples from the two-moon data. The labeled and unlabeled samples are shown in (c) and (d), and the kernel density estimation results along their x-axis projection are plotted in (a) and (b), respectively. Our approach aims to reduce the empirical distribution mismatch by aligning the sample distributions in the latent space (e) and augmenting the training samples by using interpolation between labeled and unlabeled data (f).}
	\label{fig:distmis}
	\vspace{-5pt}
\end{figure}

While many methods have been proposed to extend the traditional supervised deep learning methods to semi-supervised learning scenarios, we still lack theoretical analysis on how these methods are able to improve the learning performance, especially for the deep learning based semi-supervised learning approaches. To this end, we propose a new deep semi-supervised learning framework called Semi-supervised Learning by Empirical Distribution Alignment (SLEDA), in which we address the semi-supervised learning problem from the novel perspective of empirical distribution mismatch between labeled and unlabeled samples. On one hand, the proposed framework provides us with a novel perspective to conduct a theoretical analysis of semi-supervised learning based on the theory of distribution divergence minimization. We show that the generalization error of semi-supervised learning can be effectively bounded by minimizing the training error on labeled data and the empirical distribution distance between labeled and unlabeled data. On the other hand, it also motivates us to develop a new deep semi-supervised learning method. Based on the new framework, we design a simple and effective deep semi-supervised learning approach called Augmented Distribution Alignment that combines an adversarial training strategy and a sample interpolation strategy.

In particular, we start from the essential sampling bias issue in SSL, \ie, \textbf{\emph{the empirical distribution of labeled data often deviates from the true sample distribution}}, due to the limited sampling size of labeled data.
We illustrate this issue with the classical two-moon data in Fig.~\ref{fig:distmis}, in which we plot $6$ labeled samples in Fig.~\ref{fig:distmis}(c) and $1,000$ unlabeled samples in Fig.~\ref{fig:distmis}(d). It can be observed that the two-moon structure is captured well by the unlabeled samples. However, due to the randomness of sampling and the small sample size issue, the underlying distribution can hardly be retrieved from the labeled data, though it is also sampled from the same two-moon distribution. In terms of the empirical distribution, this deviation also leads to a considerable distribution mismatch between labeled and unlabeled data, as shown by the density estimation results along their x-axis projection (see Fig.~\ref{fig:distmis}(a) and Fig.~\ref{fig:distmis}(b)). For SSL, a similar empirical distribution mismatch is also observed in real-world datasets (see Section~\ref{sec:ea}). As shown in the recent domain adaptation works \cite{ganin2014unsupervised,long2015learning,long2019transferable}, the performance of the learnt model is often degraded significantly when it is applied to a testing set with considerable empirical data distribution mismatch. Therefore, the SSL methods could also be affected by the data distribution mismatch issue when propagating the labels from labeled data to unlabeled data.

Based on the above observations, we revisit the semi-supervised learning (SSL) problem and propose the new SLEDA framework for SSL, in which we aim to minimize the classification loss on the labeled samples, and simultaneously reduce the empirical distribution mismatch between the labeled and unlabeled samples. This framework enables us to analyze SSL from the new perspective of distribution divergence minimization. In particular, we prove that the generalization error of SSL methods can be strictly bounded by the empirical error on the labeled samples, the empirical distribution distance between labeled samples and unlabeled samples, and other constant terms. This bound immediately implies that we can tackle the SSL problem based on any supervised learning methods by additionally minimizing the empirical distribution distance between labeled samples and unlabeled samples, for which any existing technologies developed in the domain adaptation community can be readily used.

To this end, based on our framework, we propose a simple and effective SSL method called Augmented Distribution Alignment Network (our method). On one hand, as illustrated in Fig.~\ref{fig:distmis}(e), we first adopt the widely used adversarial training strategy to minimize the distribution mismatch between labeled and unlabeled data, such that the feature distributions are well aligned in the latent space. On the other hand, to address the small sample size issue and further enhance the distribution alignment, we also propose a data augmentation strategy to produce pseudo-labeled samples by generating interpolated samples between the labeled and unlabeled training sets, as illustrated in Fig.~\ref{fig:distmis}(f). In addition, benefiting from the new perspective for solving the semi-supervised learning problem, our newly proposed ADA-Net is also complementary to most existing SSL methods. Specifically, the two newly proposed strategies in our ADA-Net can be readily incorporated into the existing deep SSL methods to futher boost the performance of those state-of-the-art SSL methods.

We conduct comprehensive experiments on two benchmark datasets SVHN and CIFAR10 for semi-supervised image recognition and another two benchmark datasets ModelNet40 and ShapeNet55 for semi-supervised point cloud recognition. Extensive experimental results demonstrate the effectiveness of our proposed framework for both SSL tasks.

A preliminary conference version of this work was published in \cite{wang2019adanet}. This work extends \cite{wang2019adanet} in both theoretical and experimental aspects, which are summarized as follows,
\begin{itemize}
    \item In terms of the theoretical aspect, by analyzing the generalization bound for SSL, we prove the theoretical connection between minimizing the empirical distribution mismatch of labeled and unlabeled samples and optimizing the generalization bound for SSL with the aid of unlabeled samples. This provides a new theoretical basis for our method ADA-Net introduced in the preliminary conference version \cite{wang2019adanet}, and also provides a new perspective to understand the SSL problem..
    \item In terms of the methodology, we also propose two new variants of our ADA-Net by additionally incorporating the adversarial distribution alignment and cross-set sample augmentation strategies into two existing SSL methods to further improve their classification performance.
    \item In terms of the experimental aspect, we conduct additional experiments to validate our ADA-Net method. We extend our ADA-Net method for semi-supervised point cloud recognition, by utilizing a new mixup method specifically designed for point cloud recognition. The latest state-of-the-art SSL approaches are also implemented for comparison. This extends the SSL benchmarks, and helps the research community to validate the SSL approaches in a more comprehensive way.  
\end{itemize}

\section{Related Work}
\subsection{Semi-supervised Learning}
As a classical learning paradigm, various methods have been proposed for semi-supervised learning, including label propagation, graph regularization, co-training, \etc\cite{blum1998combining, mitchell2004role, grandvalet2005semi, blum2001learning, joachims2003transductive}. We refer interested readers to \cite{zhu2005semi} for a comprehensive survey. Recently, there is increasing interest in training deep neural networks for SSL \cite{tarvainen2017mean, laine2017temporal, miyato2018virtual, odena2018realistic, cicek2018saas, Chen_2018_ECCV, rizve2021in}. This is partially due to the data-intensive requirement of the conventional deep learning techniques, which often imposes heavy data annotation demands and brings high human annotation costs. Recently, different deep SSL methods have been proposed. For example, the works in \cite{laine2017temporal, tarvainen2017mean,miyato2018virtual} proposed to add small perturbations on unlabeled data, in order to enforce a consistency regularization~\cite{odena2018realistic} on the outputs of the model. Other works~\cite{Chen_2018_ECCV,cicek2018saas,rizve2021in} adopted the self-training idea and used the propagated labels with a memory module, or the regularization. The ensembling strategies were also used for SSL. For example, the work in \cite{laine2017temporal} used the average prediction result based on the outputs of the network-in-training over time to regularize the model, while the work in \cite{tarvainen2017mean} instead used the accumulated parameters for prediction. Meanwhile, other works \cite{berthelot2019mixmatch,berthelot2020remixmatch,sohn2020fixmatch} used carefully designed image data augmentation schemes during the training process.
On the theoretical side, several previous works have proposed different SSL theories based on the notion of compatibility \cite{balcan2005}, the cluster assumption \cite{rigollet2007}, the manifold assumption \cite{globerson2017}, the causal mechanism \cite{kuegelgen2020causality}, or the self-training strategy \cite{wei2021theoretical}.

In contrast to the above works, we revisit the SSL problem from a new perspective based on the empirical distribution mismatch, which was rarely discussed in the literature for SSL. By simply reducing the data distribution mismatch, we show that our newly proposed augmented distribution alignment strategy in combination with vanilla neural networks performs competitively with the state-of-the-arts SSL methods. Moreover, since we revisit the SSL problem from a new perspective, our approach is potentially complementary to these existing SSL approaches, and can be readily combined with them to further boost their performance. 

\subsection{Data Sampling Bias Problem}
Data Sampling bias was usually discussed in the literature for both supervised learning and domain adaptation~\cite{rosset2005method,dudik2006correcting,huang2007correcting}. Many works have been proposed to measure or reduce the data sampling bias in the training process~\cite{ganin2014unsupervised,long2015learning,long2019transferable}. Recently, in line with the generative adversarial networks~\cite{goodfellow2014generative}, the adversarial training strategy was also used to reduce the empirical data distribution mismatch for domain adaptation~\cite{ganin2014unsupervised,yan2017learning,zhang2018collaborative,zhang2019self, chen2018domain,gong2019dlow,chen2019learning}. Researchers generally assume the data from the two domains are sampled from two different distributions, while in SSL the labeled and unlabeled samples are from the same distribution.
Nonetheless, the domain adaptation techniques for reducing the domain distribution mismatch can still be readily used to address the newly proposed empirical data distribution mismatch issue in SSL. To this end, in this work, we employ the adversarial training strategy in~\cite{ganin2014unsupervised} to reduce the data distribution mismatch between the labeled and unlabeled data. 
As discussed in Section~\ref{sec:ada4ssl}, the small sample size of labeled data may bring further challenge to aligning the data distributions, for which we additionally employ a sample augmentation strategy.

\subsection{Sample Interpolation}
Our work is also related to the recently proposed interpolation-based data augmentation methods for neural network training~\cite{zhang2017mixup,verma2018manifold,chen2020pointmixup}. In particular, the \textit{Mixup} method~\cite{zhang2017mixup} proposed to generate new training samples by using convex combination of pairs of training samples and their labels. The more recent work \cite{chen2020pointmixup} extended the Mixup method to augment 3D point cloud samples, by finding one-to-one correspondence between two point cloud samples before performing interpolation. In order to address the small sample size issue when aligning the distributions, we extend the Mixup method \cite{zhang2017mixup} to SSL by using the pseudo-labels from the unlabeled samples in the interpolation process. Moreover, we also show that by using the interpolation strategies between labeled and unlabeled data, the empirical distribution of the generated samples actually becomes closer to that of the unlabeled samples.

\section{A New SSL Framework}
\label{sec:ssl_framework}
In SSL, we are given a small number of labeled training samples and a large number of unlabeled training samples. Formally, let us denote the set of labeled training data as $\cD_l=\{(\x^l_{1}, y_{1}), \ldots, (\x^l_{n}, y_{n})\}$, where $\x^l_i$ is the $i$-th sample, $y_i$ is its corresponding label, and $n$ is the total number of labeled samples. Similarly, the set of unlabeled training data can be represented as $\cD_{u}=\{\x^{u}_1, \ldots, \x^u_m\}$, where $\x^u_i$ is the $i$-th unlabeled training sample, and  $m$ is the number of unlabeled samples. Usually  $n$ is a small number, and we have $m \gg n$. The task of semi-supervised learning is to train a classifier which performs well on the test data drawn from the same distribution as the training data.

\subsection{Empirical Distribution Mismatch in SSL}\label{sec:empd}
In SSL, the labeled training samples $\cD_l$ and unlabeled training samples $\cD_u$ are assumed to be drawn from an identical distribution. However, due to the limited number of labeled training samples, a considerable empirical distribution difference often exists between the labeled and unlabeled training samples. 

More concretely, we take the two-moon data as an example to illustrate the empirical distribution mismatch problem in Fig.~\ref{fig:distmis}. In particular, the $1,000$ unlabeled samples well describe the underlying distribution (see Fig.~\ref{fig:distmis}(d)), while the labeled samples can hardly represent the two-moon distribution (see Fig.~\ref{fig:distmis}(c)). This can be further verified by projecting their distributions along the x-axis (see Fig.~\ref{fig:distmis}(a) and Fig.~\ref{fig:distmis}(b)), from which we can observe an obvious distribution difference. Actually, when performing multiple rounds of experiments to sample the labeled data, the empirical distributions of labeled data also vary significantly, due to the small sampling number.

This phenomenon was also discussed as the sampling bias problem in the literature~\cite{gretton2012kernel,gretton2009fast,kamath2015learning}. In particular, Gretton \etal~\cite{gretton2012kernel} pointed out that the difference between two sampling experiments measured by the Maximum Mean Discrepancy(MMD) criterion depends on their sampling sizes, while in SSL, the underlying distributions of labeled and unlabeled data are often assumed identical, the MMD between the labeled and unlabeled data tends to become smaller if and only if the numbers of both labeled and unlabeled samples are large, which is also described as follows,
\begin{prop}% ~\cite{gretton2012kernel}
\label{th:1}
Let us denote $\mathcal{F}$ as a class of witness functions $f$ : $\x \rightarrow \mathcal{R}$ in the reproducing kernel Hilbert space (RKHS) induced by a non-negative kernel function %$k(\cdot, \dot)$, and assume $0\leq k(\dot, \dot)\leq K$, 
with upper bound $K$, then the MMD between the labeled dataset $\cD_l$ and the unlabeled dataset $\cD_u$ can be bounded by $Pr\{ \text{MMD}[\mathcal{F}, \cD_l, \cD_u] > 2 (\sqrt{(K/n)} + \sqrt{(K/m)} +\epsilon)\} \leq 2\exp\Big\{\frac{-\epsilon^2nm}{2K(n+m)}\Big\}$
,\end{prop}
\begin{proof}
The proof can be readily derived with Theorem 7 in \cite{gretton2012kernel} by assuming the two distributions $p$ and $q$ are identical.
%\vspace{-10pt}	
\end{proof}

\begin{figure}
	\centering
	{\includegraphics[width=0.85\linewidth]{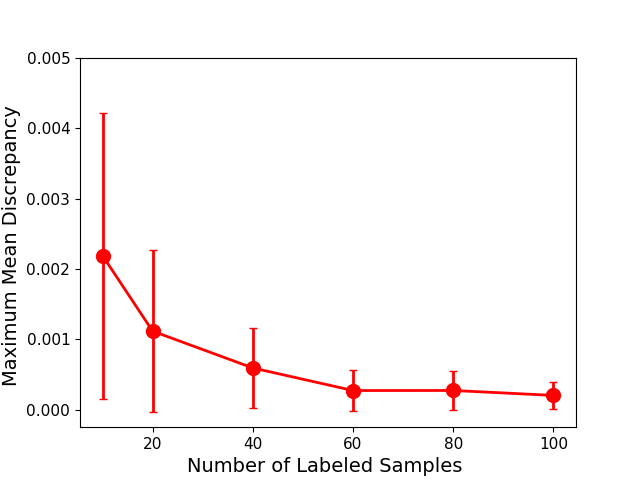}}
	\caption{MMD between the labeled and unlabeled samples in the two-moon example with varying number of labeled samples. The number of unlabeled sample is fixed as $1,000$.}
	\label{fig:twomoonsample}
\vspace{-10pt}	
\end{figure}

In SSL, the number of labeled samples $n$ is usually small, leading to a notable empirical distribution mismatch with the unlabeled samples. Specifically, we take the sampling bias problem with the two-moon data in the semi-supervised learning scenario as an example for futher illustration. In Fig.~\ref{fig:twomoonsample}, we plot the MMDs between the labeled and unlabeled samples when using different numbers of labeled samples. Fig.~\ref{fig:twomoonsample} shows that, when the number of labeled data is small, both the mean and variance of MMDs between the labeled and unlabeled samples are large, and the MMDs tend to be smaller only when $n$ becomes sufficiently large. 

The above observation implies that in SSL, the small sample size often causes empirical distribution mismatch between the labeled data and the true sample distribution. Consequentially, a model trained from the empirical distribution of the labeled data is unlikely to generalize well to the test data. While various strategies have been exploited for utilizing the unlabeled data in the conventional SSL methods~\cite{cicek2018saas, Chen_2018_ECCV}, the empirical distribution mismatch issue was rarely discussed in the SSL literature. In this work, we emphasize that it is one of the hidden factors that affects the performance of conventional SSL methods. This was also verified by the recent work \cite{odena2018realistic}, which shows that the performance of SSL methods degrades when the number of labeled samples is limited. 

%\subsection{The Learning Framework}
\subsection{The SLEDA Framework}\label{sec:sleda}
Based on the above analysis, we propose the new Semi-supervised Learning by Empirical Distribution Alignment (SLEDA) framework for SSL, to address the empirical distribution mismatch issue between the labeled and unlabeled data. In addition to training the classifier based on supervision information from labeled data, we also simultaneously minimize the distribution divergence between labeled and unlabeled data, such that the empirical distributions of labeled and unlabeled samples are well aligned in the latent space, which is illustrated in Fig.~\ref{fig:distmis}(e).

Formally, let us denote the loss function as $\ell(f(\x^l_i), y_i)$ where $f$ is the classifier to be learnt. We also define $\Omega(\cD_l, \cD_u)$ as the distribution divergence between labeled and unlabeled data measured by a pre-defined metric. Then, the objective function of our framework can be formulated as follows,
\begin{equation}
\min_{f, \Omega} \sum_{i=1}^n\ell(f(\x^l_i), y_i) + \gamma\Omega(\cD_l, \cD_u), \label{eq:ada_obj}
\end{equation}
where $\gamma$ is the trade-off parameter to balance the two terms. 

\noindent\textbf{Discussion:} 
The above SLEDA framework is very simple and generic, and we do not make any strong assumption on the terms in \eqref{eq:ada_obj}. This means that we can readily incorporate any supervised learning methods (\eg deep supervised learning methods) into the SLEDA framework. In Section~\ref{sec:ada4ssl}, we take the vanilla neural networks as an example to develop an ADA-Net method for SSL.

One potential issue in our newly proposed formulation is that due to the small number of labeled samples (\ie $n$), the optimization process for solving \eqref{eq:ada_obj} may become unstable. To address this issue, we further propose a data augmentation strategy. Inspired by the recent mixup approach \cite{zhang2017mixup} for deep supervised learning methods, we iteratively generate new training samples by using the interpolation strategy between the labeled samples and unlabeled samples, and then use the interpolated samples to learn the classifier and simultatneously reduce the empirical distribution divergence. The details are introduced in Section~\ref{sec:ada4ssl}.

\subsection{Generalization Error Bound}
Our SLEDA framework provides us a new perspective to understand the SSL problem. In this subsection, we provide rigorous theoretical analysis to show that the generalization error of SSL can be effectively bounded by the training error on labeled samples and the empirical distribution divergence between labeled and unlabeled samples. 

Let \(P\) be the unknown data distribution, from which the labeled and unlabeled training datasets \(\mathcal D_l\) and \(\mathcal D_u\) are respectively sampled. Denote the sizes of the training datasets \(\mathcal D_l\) and \(\mathcal D_u\) as \(n\) and \(m\), respectively. Let $\cH$ be a hypothesis class, and $h \in \cH$ be a hypothesis, and we also denote its empirical errors on \(\mathcal D_l\) and \(\mathcal D_u\) as \(\hat{\mathcal E}_l(h)\) and \(\hat{\mathcal E}_u(h)\), respectively. To derive the generalization error bound, we use the \emph{symmetric difference hypothesis space} $\cH\Delta\cH$ for any hypothesis space $\cH$~\cite{ben2010theory}. Following \cite{ben2010theory,kifer2004}, we represent the empirical \(\cH\Delta\cH\)-divergence  between \(\mathcal D_l\) and \(\mathcal D_u\) as \(\hat d_{\mathcal H \Delta \cH} \big( \mathcal D_l, \mathcal D_u \big) \).

\begin{theorem}
For any hypothesis space \(\mathcal H\), let us denote the labeled training sample set as \(\mathcal D_l\) and the unlabeled training sample set as \(\mathcal D_u\), where $\cD_l$ and $\cD_u$ are drawn from the same distribution $P$, and let $h \in \cH$ be a hypothesis. We further define $\hat{\mathcal E}_l(h)$ as the empirical error of $h$ on the labeled training samples, and ${\mathcal E}(h)$ as the generalization error on the unseen testing data drawn from the distribution $P$, then the generalization bound of $h$ can be described by the following inequality with probability at least \(1 - \delta\),
\begin{align}
\mathcal E(h) \leq \hat{\mathcal E}_l(h) + 
    \scalemath{0.9}{\frac 1 2 \hat d_{\mathcal H\Delta \mathcal H}\big(\mathcal D_l, \mathcal D_u\big)
		+ \sqrt{\frac 1{2m}\ln\frac 2\delta}}
\end{align}
where $\hat d_{\mathcal H\Delta \mathcal H}\big(\mathcal D_l, \mathcal D_u\big)$ is the empirical $\cH\Delta\cH$-divergence between $\cD_l$ and $\cD_u$.
\end{theorem}

Theorem 1 reveals that the generalization error of SSL  can be bounded by the training error on a small number of labeled training samples, the empirical $\cH\Delta\cH$-divergence between $\cD_l$ and $\cD_u$, and a minor term  \emph{w.r.t.} the number of unlabeled samples $m$. Note that in the generalization bound for supervised learning, the minor term is usually related to the number of labeled training samples $n$. When $n$ is small (\ie the case in SSL), the generalization error of the supervised learning methods cannot be effectively bounded by the training error, as the minor term tends to be large. In the generalization bound shown in Theorem 1, we can control the minor term $\sqrt{\frac 1{2m}\ln\frac 2\delta}$ related to the number of unlabeled samples $m$, as $m$ is often much larger than $n$, so the minor term would be small. The generalization bound can also become tighter by reducing the empirical $\cH\Delta\cH$-divergence $\hat d_{\mathcal H\Delta \mathcal H}\big(\mathcal D_l, \mathcal D_u\big)$, which can be achieved with the adversarial distribution alignment strategy as analyzed in \cite{ganin2016domain}. Below we prove Theorem 1.

\begin{proof}
Given an arbitrary hypothesis \(h\), based on the Hoeffding's inequality (see page 58 in \cite{abu-mostafa2012learning}), we arrive at:
\begin{align}
\scalemath{0.9}{Pr\Bigg(\big| \mathcal E(h) - \hat{\mathcal E}_l(h) \big| 
	\leq \sqrt{\frac 1{2n}\ln\frac 2 \delta} \Bigg)
	\geq 1 - \delta
	} \label{eq:hoeffding1_ver2}\\
\scalemath{0.9}{Pr\Bigg(\big| \mathcal E(h) - \hat{\mathcal E}_u(h) \big| 
	\leq \sqrt{\frac 1{2m}\ln\frac 2 \delta} \Bigg)
	\geq 1 - \delta
	}\label{eq:hoeffding2_ver2}
\end{align}
where $Pr(\cdot)$ denotes the probability that the event happens. 

While in practice the term \(\hat{\mathcal E}_u(h)\) related to the unlabeled dataset in \eqref{eq:hoeffding2_ver2} cannot be computed, we still use it as an intermediate term for deriving the generalization bound. In particular, to connect \eqref{eq:hoeffding1_ver2} and \eqref{eq:hoeffding2_ver2}, we have,
\begin{align}\label{eq:triangle_supremum_ver2}
\scalemath{0.9}{\big| \mathcal E(h) - \hat{\mathcal E}_l(h) \big|}
	&\scalemath{0.9}{=\big| \mathcal E(h) - \hat{\mathcal E}_u(h) 
	+ \hat{\mathcal E}_u(h) - \hat{\mathcal E}_l(h) \big|}
	\nonumber\\
	%\scalemath{1}{\big| \mathcal E(f) - \hat{\mathcal E}_l(f) \big|} 
    &\scalemath{0.9}{\leq\big|\mathcal E(h) - \hat{\mathcal E}_u(h) \big|
	+ \big|\hat{\mathcal E}_l(h) - \hat{\mathcal E}_u(h) \big|}
\end{align}

Let us denote $\big| \mathcal E(h) - \hat{\mathcal E}_u(h) \big| 
	\leq \sqrt{\frac 1{2m}\ln\frac 2 \delta}$ as event A, and $\big|\mathcal E(h) - \hat{\mathcal E}_l(h) \leq \big| \hat{\mathcal E}_l(h) - \hat{\mathcal E}_u(h) \big| + \sqrt{\frac 1{2m}\ln\frac 2 \delta}$ as event B.
From \eqref{eq:triangle_supremum_ver2}, we know that if event
 A happens, then we have  $\big|\mathcal E(h) - \hat{\mathcal E}_l(h) \leq \big|\mathcal E(h) - \hat{\mathcal E}_u(h) \big|
	+ \big|\hat{\mathcal E}_l(h) - \hat{\mathcal E}_u(h) \big| \leq \big| \hat{\mathcal E}_l(h) - \hat{\mathcal E}_u(h) \big| + \sqrt{\frac 1{2m}\ln\frac 2 \delta}$, which indicates event B definitely happens. So the probability that event B happens would be no less than the probability that event A happens, \ie,

\begin{align}\label{eq:bound_u_ver2}
&\scalemath{0.9}{Pr\Bigg(\big|\mathcal E(h) - \hat{\mathcal E}_l(h)\big| \leq
	\big|\hat{\mathcal E}_l(h) - \hat{\mathcal E}_u(h)\big|
	    + \sqrt{\frac 1{2m} \ln\frac 2 \delta} \Bigg)}\nonumber\\
\scalemath{0.9}{\geq}& \scalemath{0.9}{Pr\Bigg(\big| \mathcal E(h) - \hat{\mathcal E}_u(h) \big| 
	\leq \sqrt{\frac 1{2m} \ln\frac 2 \delta} \Bigg)}
	\nonumber\\
\scalemath{0.9}{\geq} & \scalemath{0.9}{1 - \delta}
\end{align}

Next, we will bound the term \(\big|\mathcal E(h) - \hat{\mathcal E}_l(h)\big|\) in \eqref{eq:bound_u_ver2} with the empirical \(\mathcal H \Delta \cH\)-divergence.

\begin{align}\label{eq:bound_lemma2variant2}
\scalemath{0.9}
{\big|\hat{\mathcal E}_l(h) - \hat{\mathcal E}_u(h)\big|
    \leq \sup_{h\in\mathcal H}\big|\hat{\mathcal E}_l(h) - \hat{\mathcal E}_u(h)\big|
    = \frac 1 2 \hat d_{\mathcal H \Delta \mathcal H}(\mathcal D_l, \mathcal D_u)
}
\end{align}
Note the second equality in \eqref{eq:bound_lemma2variant2} holds due to the definition of \(\hat d_{\mathcal H\Delta\mathcal H}(\mathcal D_l, \mathcal D_u)\) \cite{ben-david2006analysis,ben2010theory}.

Finally, we can prove the generalization bound in Theorem 1 by substituting the inequality in \eqref{eq:bound_lemma2variant2} into the inequality in \eqref{eq:bound_u_ver2}. Here we complete the proof.
\end{proof}
% END OF LEMMA 1 PROOF VERSION II

When $n\ll m$, the minor term $\sqrt{\frac 1{2m} \ln\frac 2 \delta}$ would be much smaller than  $\sqrt{\frac 1{2n} \ln\frac 2 \delta}$ in \eqref{eq:hoeffding1_ver2}. Therefore, for the SSL task, it is likely that we can achieve a tighter bound than \eqref{eq:hoeffding1_ver2} by minimizing the empirical $\cH\Delta\cH$-divergence.

\begin{figure*} [h!]
	\centering
	\includegraphics[width=0.9\linewidth]{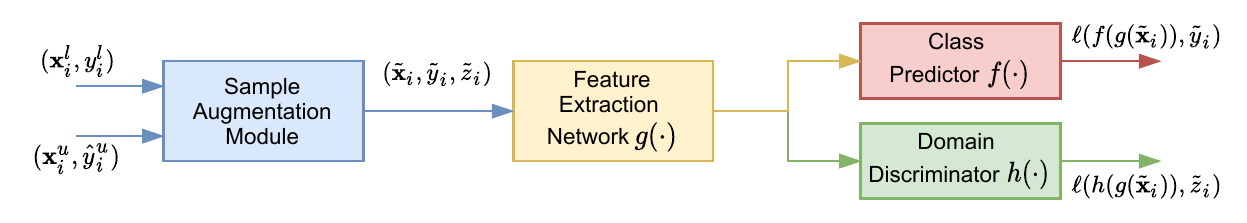}
	\vspace{-10pt}
	\caption{The network architecture of our proposed ADA-Net, in which we introduce an additional discriminator branch $h(\cdot)$ with a gradient reverse layer (GRL) to the vanilla neural network. $\mathbf x^l$ and $y^l$ denote the labeled sample and its class label, $\mathbf x^u$ and $\hat y^u$ denote the unlabeled sample and its pseudo-label produced by the class predictor, respectively, $\tilde{\x}$, $\tilde y$, and $\tilde z$ denote the augmented sample, its class label produced by the class predictor, and its domain label (\ie, 0 for the labeled domain and 1 for the unlabeled domain), respectively, and $\ell(\cdot,\cdot)$ denotes the cross-entropy loss function. During the training process, the cross-set sample interpolation strategy is used between labeled and unlabeled samples, and we feed the interpolated samples into the network. Pseudo-labels of unlabeled samples are obtained by using the prediction scores of the class predictor trained in the last iteration (see Section~\ref{sec:summary} for more details).}
	\label{damain}
\end{figure*}

\section{Augmented Distribution Alignment Network (ADA-Net) for SSL}\label{sec:ada4ssl}
Based on our SLEDA framework, in this section, we develop a new deep SSL method called \emph{Augmented Distribution Alignment Network (ADA-Net)}. In our ADA-Net, we use two strategies \emph{adversarial distribution alignment} and \emph{cross-set sample augmentation}, to respectively tackle the empirical distribution mismatch and small sample size issues. 

\subsection{Adversarial Distribution Alignment}
\label{sec:ada}
We first discuss how to minimize the empirical distribution mismatch between the labeled and unlabeled samples. As shown in~\cite{ganin2014unsupervised,ganin2016domain}, under the mild assumptions, minimizing the $\cH\Delta\cH$-divergence can be implemented by using an adversarial training strategy, which is equivalent to minimizing the empirical $\cH$-divergence $\hat d_{\cH}(\cdot,\cdot)$. Therefore, we employ $\hat d_{\cH}(\cdot,\cdot)$ to measure the distribution divergence $\Omega(\cdot, \cdot)$ in Eq.~\eqref{eq:ada_obj} for aligning the empirical distributions between labeled and unlabeled data.

In particular, let us denote $g(\cdot)$ as a feature extractor (\eg, the convolutional layers), which maps sample $\x$ into a latent feature space. Moreover, let $h:g(\x)\rightarrow\{0,1\}$ be a binary discriminator which outputs 0 for the labeled samples and 1 for the unlabeled samples, and denote the class of binary discriminators as \newcontent{$\mathcal H$}. By Lemma 2 in \cite{ben2010theory}, the empirical $\cH$-Divergence between the labeled and unlabeled samples can be written as:
\begin{eqnarray}
\hat d_{\mathcal H}(\cD_l, \cD_u) \!\!= \!\!2 \Big\{1 - \min_{h\in\mathcal H}\left[err(h, g, \cD_l)+err(h, g, \cD_u)\right]\Big\}
\end{eqnarray}
where $err(h, g, \cD_l) = \frac{1}{n}\sum_{\x^l}[h(g(\x^l))\neq 0]$ is the prediction error of the discriminator $h$ on the labeled samples, and $err(h, g, \cD_u)$ is similarly defined for the unlabeled samples. 

Intuitively, when the empirical distribution mismatch is large, the discriminator can easily distinguish between labeled and unlabeled samples. Thus, its prediction errors would be smaller, and the $\cH$-divergence is higher, and vice versa. Therefore, to reduce the empirical distribution mismatch between labeled and unlabeled samples, we minimize the distribution distance $\hat d_{\cH}(\cD_l, \cD_u)$ to enforce the feature extractor $g$ to generate a good latent space, in which the two sets of features are well aligned. To achieve this goal, we solve the following optimization problem:
\begin{eqnarray}
\min_{g}{\hat d_{\cH}}(\cD_l, \cD_u) \!\!= \!\!\max_g\min_{h\in\mathcal H}\left[err(h, g, \cD_l)+err(h, g, \cD_u)\right] \!\! \!\!
\end{eqnarray}

The above max-min problem can be optimized with the adversarial training methods. In~\cite{ganin2016domain}, Ganin and Lempitsky showed that it can be implemented as a simple gradient reverse layer (GRL), which automatically reverses the gradient the domain discriminator before back-propagating the gradient to the feature extraction block. With the newly introduced GRL, we can directly minimize the classification loss from the domain discriminator $h$ with the standard optimization library. 

\subsection{Cross-Set Sample Augmentation}\label{sec:sample_augmentation}
As mentioned in Section~\ref{sec:ssl_framework}, in SSL, the optimization process may become unstable due to the limited sampling size of labeled data, which eventually degrades the classification performance. In order to assist the alignment process, and as inspired by \cite{zhang2017mixup}, we propose to generate new training samples by using the interpolation strategy between labeled and unlabeled samples (\ie the cross-set sample interpolation strategy). 
In particular, for each $\x^u_i$, we assign a pseudo-label $\hat{y}^u_i$, which is generated by using the prediction score from the model trained in the previous iteration. Then, given a labeled sample $\x^l_i$ and an unlabeled sample $\x^u_i$, the interpolated sample can be represented as,
\begin{eqnarray}
\tilde{\x}_i &=& \lambda_i \x^l_i + (1-\lambda_i) \varphi_i(\x^u_i), \label{eqn:interp_x}\\
\tilde{y}_i &=& \lambda_i y^l_i + (1-\lambda_i) \hat{y}^u_i,\label{eqn:interp_y}\\
\tilde{z}_i &=& \lambda_i \cdot 0 + (1-\lambda_i)\cdot1\label{eqn:interp_z},
\end{eqnarray}
where $\lambda_i$ is a random variable that is generated from a prior $\beta$ distribution, \ie $ \lambda_i \sim \beta(\alpha, \alpha)$ with $\alpha$ being a hyper-parameter to control the shape of the $\beta$ distribution, $\varphi_i$ is the input transformation function for $\x^u_i$, $\tilde{\x}_i$ is the interpolated sample, $\tilde{y}_i$ is its class label, and $\tilde{z}_i$ is its label for training the discriminator. We define $\varphi_i$ in different ways for the image recognition task and the point cloud recognition task as follows. 

\paragraf{The Choice of \(\varphi_i\):} For image recognition, \(\varphi_i\) is simply the identity operator whose output is exactly the same as the input. For point cloud recognition, we cannot use the identity operator, since the 3D points are unordered and there is no one-to-one correspondence between the points of two point clouds. To establish the one-to-one correspondence between two point clouds, we adopt the method proposed in \cite{chen2020pointmixup}, which finds the optimal matching between the points of two point clouds by solving the assignment problem defined as the Earth Mover's Distance (EMD). Thus, for point cloud recognition, we define $\varphi_i$ in \eqref{eqn:interp_x} as the permutation operator, which reorders the points of $\x^u_i$ based on the one-to-one correspondence between the points of $\x^l_i$ and $\x^u_i$. In our implementation, we solve the assignment problem efficiently by using the Auction Algorithm~\cite{bertsekas1992auction}.

The advantages of such cross-set sample augmentation strategy are two-fold. Firstly, the interpolated samples greatly enlarge the training data set, making the learning process more stable, especially for learning the deep models. It was also shown in \cite{zhang2017mixup} that such data augmentation strategy is helpful for improving model robustness. 

Secondly, each pseudo-sample is generated by using the interpolation strategy between a pair of labeled and unlabeled samples, thus the distribution of pseudo-samples is expected to be closer to the real distribution than that of the original labeled training samples. This advantage was theoretically shown in our previous work \cite{wang2019adanet} by employing the Euclidean generalized energy distance \cite{szekely2013energy}, and it implies that the generated pseudo-labeled samples can be deemed as being sampled from the intermediate distributions between the two empirical distributions of the labeled and unlabeled data. As shown in the previous domain adaptation works~\cite{gopalan2011domain,gong2012geodesic}, it is beneficial to use such intermediate distributions to reduce the gap between the two distributions, and learn more robust models.

\subsection{Our Basic ADA-Net Method based on Vanilla Networks}\label{sec:summary}
We unify the adversarial distribution alignment and cross-set sample augmentation strategies into one framework, which finally leads to our augmented distribution alignment approach. 

In Figure~\ref{damain}, we show an example of how to incorporate our augmented distribution alignment approach into a vanilla convolutional neural network, referred to as \emph{ADA-Net}. Specifically, in addition to the classification branch, we add several fully-connected layers as the domain discriminator $h$ to distinguish the labeled domain and the unlabeled domain (see Section~\ref{sec:ada} for more details). Then, for each mini-batch, we use the cross-set sample augmentation strategy (see \eqref{eqn:interp_x}, \eqref{eqn:interp_y}, \eqref{eqn:interp_z}) to generate the interpolated samples and labels, and use them as the training data to learn our ADA-Net. Let us denote the parameters for the feature extractor $g$, the class predictor $f$, and the domain discriminator $h$ as $\bm{\theta}_g$, $\bm{\theta}_f$, and $\bm{\theta}_h$, respectively. 
The objective function for training the network can be formulated below by replacing the training samples
% with the interpolated samples 
and the term $\Omega(\cdot, \cdot)$ in (\ref{eq:ada_obj}), \ie,
\begin{align} \label{eq:ada_obj_final_0}
\max_{\bm{\theta}_h}\min_{\bm{\theta}_f, \bm{\theta}_g} \sum_i
&\big\{\lambda_i\ell(f(g(\tilde{\x}_i; \bm{\theta}_g); \bm{\theta}_f), \tilde{y}_i) \nonumber\\
&- \gamma\ell(h(g(\tilde{\x}_i; \bm{\theta}_g); \bm{\theta}_h), \tilde{z}_i)\big\}
\end{align}
where $\ell(\cdot, \cdot)$ is the loss function and we use the cross-entropy loss in this work,  and $\lambda_i$ is the weight for balancing the classification losses which corresponds to the value $\lambda_i$ for generating the interpolated sample $\tilde{\x}_i$ (see (\ref{eqn:interp_x})). The explanation for introducing this weight is as follows. The higher the value $\lambda_i$ is, the higher the proportion of $\tilde{\x}_i$ coming from the labeled set is, and the more confident we are about its label $\tilde{y}_i$, and vice versa.

To implement the adversarial training strategy, we add a gradient reverse layer \cite{ganin2016domain} before the domain discriminator, which automatically reverses the sign of the gradient from the domain discriminator during back-propagation. Then we can rewrite the objective function in \eqref{eq:ada_obj_final_0} as follows,
\begin{align} \label{eq:ada_obj_final}
\min_{\bm{\theta}_f, \bm{\theta}_g, \bm{\theta}_h} \sum_i
&\big\{\lambda_i\ell(f(g(\tilde{\x}_i; \bm{\theta}_g); \bm{\theta}_f), \tilde{y}_i) \nonumber\\
&+ \gamma\ell(h(g(\tilde{\x}_i; \bm{\theta}_g); \bm{\theta}_h), \tilde{z}_i) \big\}
\end{align}

We depict the training pipeline in Algorithm~\ref{algo:alg}. Aside from the simple sample interpolation strategy, the network can be optimized with the standard back-propagation approaches. Therefore, our augmented distribution alignment strategy can be easily incorporated into any existing neural network by additionally introducing a domain discriminator (implemented with the GRL layer), and using the proposed cross-set sample augmentation strategy during mini-batch data preparation. 

\begin{algorithm}[t]
\SetAlgoLined
\SetKwInOut{Input}{Input}
\SetKwInOut{Output}{Output}
\caption{The training process of our ADA-Net.}
\label{algo:alg}
\Input{A batch of labeled samples \{$(\x^l, y^l)\}$, a batch of unlabeled samples $\{\x^u\}$}
\begin{enumerate}
  \item Perform one forward step to produce the pseudo-labels for the unlabeled samples, \ie, $\hat y^u \xleftarrow{} f(g(\x^u))$
  \item Sample $\lambda$ from a prior $\beta$ distribution for each pair of labeled/unlabled samples in the batch, and generate a batch of pseudo-labeled samples $\{(\tilde \x, \tilde y, \tilde z)\}$ using \\(\ref{eqn:interp_x}), (\ref{eqn:interp_y}), (\ref{eqn:interp_z}).
  \item Train the model by minimizing the objective function \\ in \eqref{eq:ada_obj_final}.
\end{enumerate}
\Output{The feature extractor $g$, the class predictor $f$, and the domain discriminator $h$.}
\end{algorithm}

\subsection{Our ADA-Net in Combination with Two SSL Methods}\label{sec:sleda_ext}
In Section~\ref{sec:summary}, we have proposed the ADA-Net method, which incorporates our newly proposed strategies into the vanilla neural network to address the SSL problem by considering the empirical distribution alignment issue. The empirical distribution alignment strategy can also be used to further improve the existing SSL methods, as these SSL methods rarely considered the empirical distribution mismatch issue between the labeled and unlabeled samples. In particular, our adversarial distribution alignment and cross-set sample augmentation strategies in ADA-Net can be easily incorporated into existing semi-supervised methods to further boost their performance. More specifically, we present the general objective function of our extended ADA-Net method by additionally introducing a new loss function $\mathcal L^u_i$ for the $i$-th unlabeled sample into \eqref{eq:ada_obj_final}, \ie,
\begin{align}\label{eq:ada_ext}
\min_{\bm{\theta}_f, \bm{\theta}_g, \bm{\theta}_h} \sum_i 
&\big\{\lambda_i\ell(f(g(\tilde{\x}_i; \bm{\theta}_g); \bm{\theta}_f), \tilde{y}_i) \nonumber\\
&+ \gamma\ell(h(g(\tilde{\x}_i; \bm{\theta}_g); \bm{\theta}_h), \tilde{z}_i) + \mathcal L^u_i\big\}
\end{align}

Below we take two SSL methods (\ie VAT+Ent~\cite{miyato2018virtual} and ICT~\cite{verma2019interpolation}) as the examples to discuss how the existing SSL methods can be incorporated into our SLEDA framework, in which we will focus on the definition of $\mathcal L^u_i$.

\subsubsection{Our ADA-Net (VAT+Ent)}
The VAT+Ent~\cite{miyato2018virtual} method utilizes the unlabeled samples by enforcing consistency between the class prediction score of each original unlabeled sample and that of this unlabeled sample after adversarial perturbation (see~\cite{miyato2018virtual} for more details). Specifically, following the VAT+Ent~\cite{miyato2018virtual} method, we minimize the following objective loss function for the $i$-th unlabeled sample,
\begin{align}\label{eq:ada_vat_obj}
\!\!\!\!\!\!\!\!\mathcal L^u_i\!=\!\lambda_v D(q^u_i, f(g(\bar{\x}^u_i; \bm{\theta}_g); \bm{\theta}_f))\!+\!\lambda_e H(f(g(\x^u_i; \bm{\theta}_g); \bm{\theta}_f))\!\!\!\!\!
\end{align}
where $D(\cdot, \cdot)$ is the consistency criterion implemented with the KL-divergence function in~\cite{miyato2018virtual}, $q^u_i$ is the class prediction score of the $i$-th unlabeled sample produced by using the latest feature extractor and class predictor from the last training iteration, $\bar{\x}^u_i$ is the same unlabeled sample after adversarial perturbation (see~\cite{miyato2018virtual}), $H(\cdot)$ is the conditional entropy loss function which enforces the class prediction score of $\x^u_i$ to become close to a one-hot vector, and $\lambda_v$ and $\lambda_e$ are the trade-off parameters. The training process of our ADA-Net (ICT) is almost the same as Algorithm~\ref{algo:alg}, except that in Step (3) we train the model by minimizing the objective loss function in \eqref{eq:ada_ext}, in which the loss $\cL_u^i$ is defined in \eqref{eq:ada_vat_obj}.

\subsubsection{Our ADA-Net (ICT)}
The ICT~\cite{verma2019interpolation} method introduces a consistency criterion on the unlabeled samples, which is based on the sample Mixup strategy~\cite{zhang2017mixup}. Specifically, they used a mean teacher model, where the mixture of the output from any two unlabeled samples by the teacher model is enforced to be consistent with the output of the mixed sample by the student model. However, the empirical distribution mismatch between labeled and unlabeled samples is not discussed in ICT~\cite{verma2019interpolation}. 

In our ADA-Net method in combination with ICT~\cite{verma2019interpolation}, we introduce the distribution alignment module to improve the results of ICT. Specifically, we introduce a domain discriminator for the student model. The same as in ICT, the teacher model is a moving average of the student model. We use $g$, $f$, and $h$ to denote the feature extractor, the class predictor, and the domain discriminator in the student model, respectively. Then we adopt the consistency loss function for the $i$-th unlabeled sample in ICT~\cite{verma2019interpolation} as the loss function $\mathcal L^u_i$ in \eqref{eq:ada_ext}, \ie,
\begin{align}\label{eq:ada_ict_obj}
    \mathcal L^u_i = w_{it} \ell_{MSE} (f(g(\tilde \x^u_i; \bm{\theta}_g); \bm{\theta}_f), \tilde y^u_i)
\end{align}
where $\ell_{MSE}(\cdot,\cdot)$ is the Mean Square Error function, $w_{it}$ is the trade-off parameter that gradually increases after each training iteration as suggested in \cite{verma2019interpolation}, and $(\tilde{\x}^u_i, \tilde y^u_i)$ is the augmented sample produced by using the sample interpolation strategy on the unlabeled samples (\ie, the within-set sample interpolation strategy). When computing $(\tilde{\x}^u_i, \tilde y^u_i)$, we need to first compute the pseudo-label $\hat y^u_i$ of each $\x^u_i$, then we use the sample interpolation strategy for each pair of unlabeled samples $(\x^u_i, \hat y^u_i)$ and $(\x^u_j, \hat y^u_j)$ in the same way as \eqref{eqn:interp_x}~\eqref{eqn:interp_y}. When generating $(\tilde{\x}_i, \tilde y_i)$ with the cross-set sample interpolation strategy discussed in Section~\ref{sec:sample_augmentation} and generating $(\tilde{\x}^u_i, \tilde y^u_i)$ with the aforementioned within-set sample interpolation strategy, we use the teacher model to produce the pseudo-label $\hat y^u_i$ for each $\x^u_i$. The training process of our ADA-Net (ICT) is almost the same as Algorithm~\ref{algo:alg}, except that in Step (3) we train the model by minimizing the objective loss function in \eqref{eq:ada_ext}, in which the loss $\cL_u^i$ is defined in \eqref{eq:ada_ict_obj}.

\section{Experiments}
In this section, we first conduct experiments on the benchmark image datasets SVHN and CIFAR10 to evaluate our proposed ADA-Net for the semi-supervised image recognition task. Then we evaluate our ADA-Net for the semi-supervised point cloud recognition task on the benchmark point cloud datasets ModelNet40 and ShapeNet55.

\subsection{Semi-Supervised 2D Image Recognition}
\subsubsection{Datasets}
\paragraf{SVHN:} The Street View House Numbers (SVHN) dataset~\cite{netzer2011reading} is a dataset of real-world digital photos. It consists of 10 classes and 73,257 training images with the resolution of 32$\times$32. Following~\cite{miyato2018virtual}, out of the full training set, 1000 images are used together with their labels for supervised learning. The rest of the training images are used as the unlabeled samples. Random translation is the only augmentation strategy used for this dataset.

\paragraf{CIFAR10:} The CIFAR10 dataset~\cite{krizhevsky2009learning} has 10 classes, and consists of 50,000 training images as well as 10,000 testing images. All images have the resolution of 32$\times$32. 4,000 samples from the training images are used as the labeled set in our experiments, and the other training images are used as the unlabeled samples.

\subsubsection{Implementation Details}
\renewcommand{\thefootnote}{\arabic{footnote}}
We follow \cite{tarvainen2017mean} to use the Conv-Large network as the the backbone network, and implement our ADA-Net in Tensorflow~\cite{abadi2016tensorflow}. For the class predictor, a single fully-connected layer is used to map the features to logits. For the domain discriminator, two fully-connected layers, each with 1,024 output units, together with another fully-connected layer are used to produce two channels of soft domain labels.

The batch size is 128. The learning rate is initialized as 0.001, and then starts to linearly decay at a predefined number of training epochs, where one epoch is defined as one iteration over all unlabeled data. On CIFAR10, the network is trained for 2000 epochs and the learning rate decays after 1500 epochs. On SVHN, the network is trained for 500 epochs and the learning rate decays after 460 epochs. We use the ADAM optimizer with the momentum set as 0.9. The interpolation parameter $\alpha$ is set as 1.0 and 0.1 on CIFAR10 and SVHN, respectively, which are the same values as used in our preliminary conference work~\cite{wang2019adanet} when employing PreAct-ResNet-18~\cite{he2016identity} as the backbone network\footnotemark. The experiments on CIFAR10 and SVHN share the exactly same network and training protocol.
\footnotetext{In our preliminary conference work~\cite{wang2019adanet}, we use PreAct-ResNet-18 as the backbone network, and our method achieves the error rates of 8.87\% and 5.90\% on CIFAR10 and SVHN, respectively. Here we use Conv-Large as another backbone network to further demonstrate the effectiveness of our ADA-Net.}

\subsubsection{Experimental Results}
We summarize the classification error rates on the SVHN and CIFAR10 datasets in Table~\ref{tab:ssl-ablation}. We also include the baseline method Conv-Large~\cite{tarvainen2017mean} that is trained based on the labeled data only. To validate the effectiveness of the two strategies in our ADA-Net, we report the results from two variants of our proposed approach. In the first variant, we do not use the cross-set sample augmentation strategy and only perform the distribution alignment by using the original labeled and unlabeled samples. In the second variant, we remove the discriminator and only perform cross-set sample augmentation for learning the classifier.

\begin{figure*}
	\centering
	\includegraphics[width=0.8\linewidth]{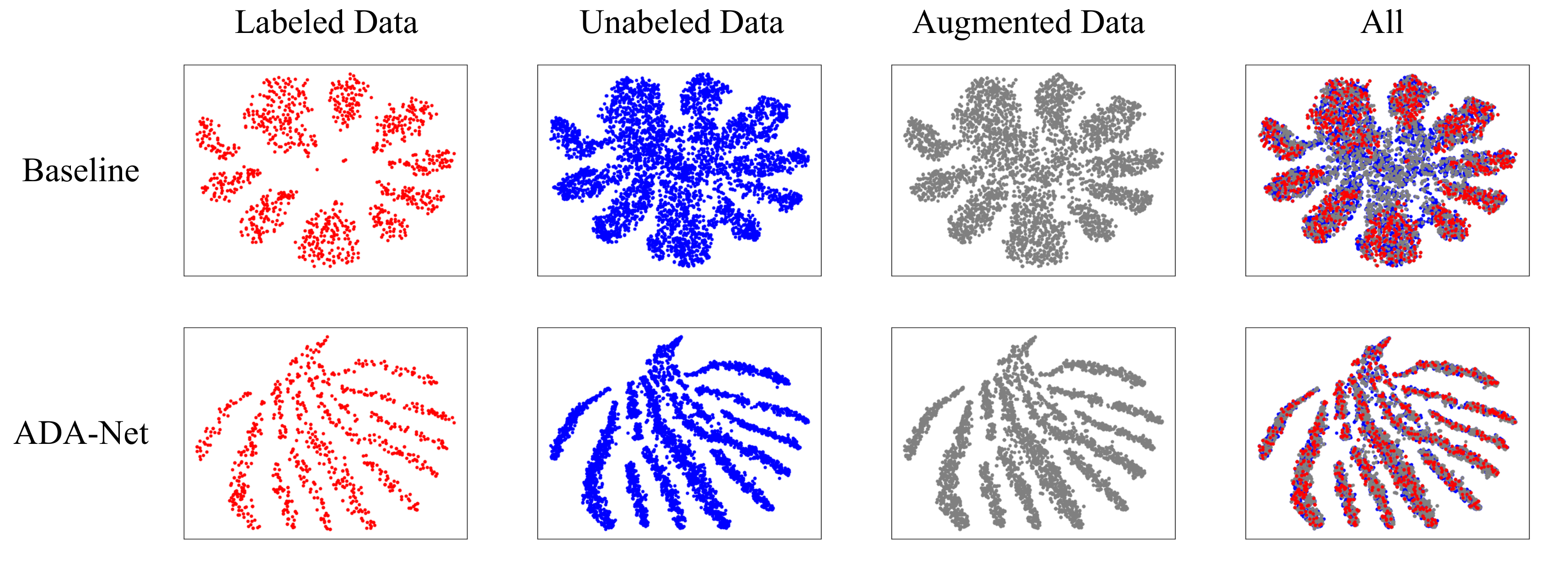}
	\vspace{-5pt}
	\caption{Visualization of the features extracted by the baseline Conv-Large method and our ADA-Net on SVHN, in which we use the t-SNE method~\cite{maaten2008tsne} to visualize the features from the labeled samples, unlabeled samples, interpolated samples, and all samples, respectively. For the baseline Conv-Large method, empirical distribution mismatch between the labeled and unlabeled samples can be observed, and the augmented samples bridge the distribution gap to some extent. For our ADA-Net, with the augmented distribution alignment strategy, empirical data distribution mismatch is well reduced.}
	\label{fig:tsne}
\end{figure*}

As shown in Table~\ref{tab:ssl-ablation}, our ADA-Net significantly improves the classification performance of the supervised baseline method Conv-Large on both datasets. We also observe that both distribution alignment and cross-set sample augmentation strategies are important for improving the performance on CIFAR10 and SVHN. When using Conv-Large as the backbone network, the distribution alignment strategy on the labeled and unlabeled datasets brings 4.91\% and 4.82\% improvement, respectively, and the cross-set sample augmentation strategy achieves 4.93\% and 0.98\% improvement, respectively. By integrating both strategies, the classification error rate can be reduced from 20.65\% and 10.82\% (the supervised baseline method) to 10.30\% and 4.62\% (our method) on the CIFAR10 and SVHN datasets, respectively.
The experimental results clearly demonstrate the effectiveness of the two newly proposed strategies in our ADA-Net method.

\begin{table}[h!]
	\begin{center}
		\caption{Classification error rates (\%) of our proposed ADA-Net and its variants as well as the baseline method Conv-Large~\cite{tarvainen2017mean} on the CIFAR10 and SVHN datasets. ``DAS" denotes the distribution alignment strategy, and ``SAS" denotes the cross-set sample augmentation strategy. Conv-Large is used as the backbone network in all methods. The best results are highlighted in boldface.}
		\label{tab:ssl-ablation}
		\begin{tabular}{c|c c|c|c} 
		\hlineB{3}
			                    & {DAS} & {SAS} & {CIFAR10} & {SVHN}\\
			\hline \hline 
			Supervised baseline & &             &20.65   &10.82\\
			\hline
			\multirow{3}{*}{Ours} & \checkmark& &15.74   &6.00\\
			\cline{2-5}
			                    & & \checkmark  &15.72   &9.84\\
			\cline{2-5}
			     &\checkmark &\checkmark        &\textbf{10.30}&\textbf{4.62}\\
			\hline
		\end{tabular}
	\end{center}
\end{table}

\subsubsection{Experimental Analysis} \label{sec:ea}
\paragraf{Feature visualization:} To better understand how our ADA-Net works, we use the convolutional block of the Conv-Large network \cite{tarvainen2017mean} as the feature extractor, and adopt the t-SNE method~\cite{maaten2008tsne} to visualize the labeled samples, the unlabeled samples, and the generated pseudo-labeled samples on the SVHN dataset in Fig.~\ref{fig:tsne}. For comparison, we also visualize the features extracted by using the baseline Conv-Large backbone trained with only labeled data. As shown in Fig.~\ref{fig:tsne}, for the baseline Conv-Large method, we observe a considerable distribution difference between the labeled and unlabeled samples, and between the labeled samples and the pseudo-labeled samples. Nevertheless, with our ADA-Net, the distributions of the samples using these three types of features are similar since we explicitly align the distributions of labeled and unlabeled samples in the training process.

\paragraf{Feature distribution:} To further demonstrate the effectiveness of our ADA-Net for reducing the data distribution mismatch, we take the first three activations of the features extracted by the baseline Conv-Large method and our ADA-Net as the examples, and plot the distributions of labeled and unlabeled samples for each dimension individually. The distribution is produced by performing kernel density estimation~\cite{parzen1962estimation, rosenblatt1956remarks} based on the samples along each individual dimension of each type of feature. As shown in Figure~\ref{fig:damaind}, we observe a considerable data distribution mismatch between the estimated empirical distributions of labeled and unlabeled samples for the baseline Conv-Large method, while the distribution mismatch is well reduced by our ADA-Net method.

\begin{figure}[h]
    \vspace{-5pt}	
	\centering
	\includegraphics[width=0.8\linewidth]{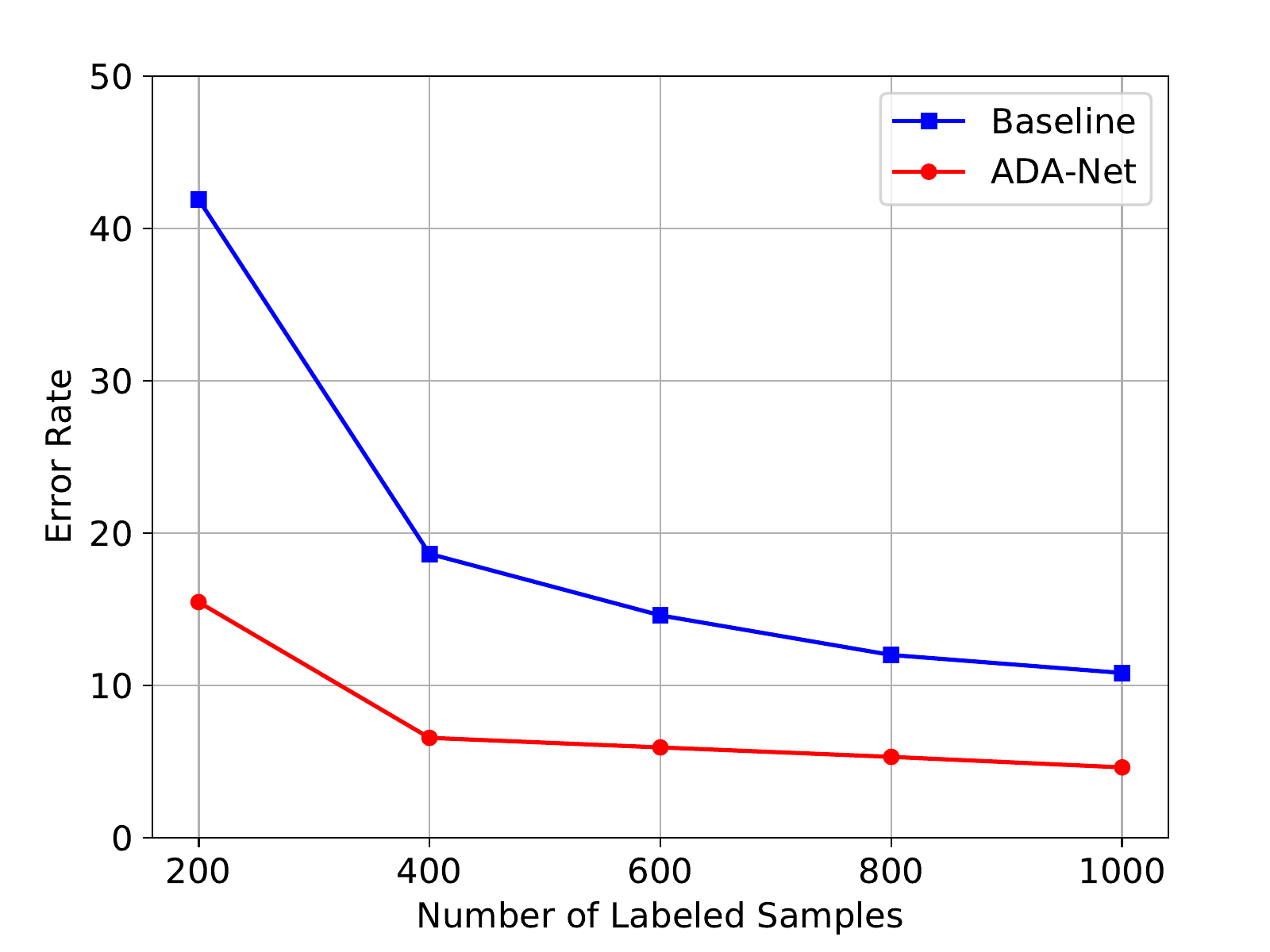}
	\caption{Classification error rates (\%) of our ADA-Net and the baseline Conv-Large method on SVHN when varying the number of labeled samples.}
	\vspace{-0cm}
	\label{fig:size}
\end{figure}
\begin{figure}[t]
	\centering
	\includegraphics[width=\linewidth]{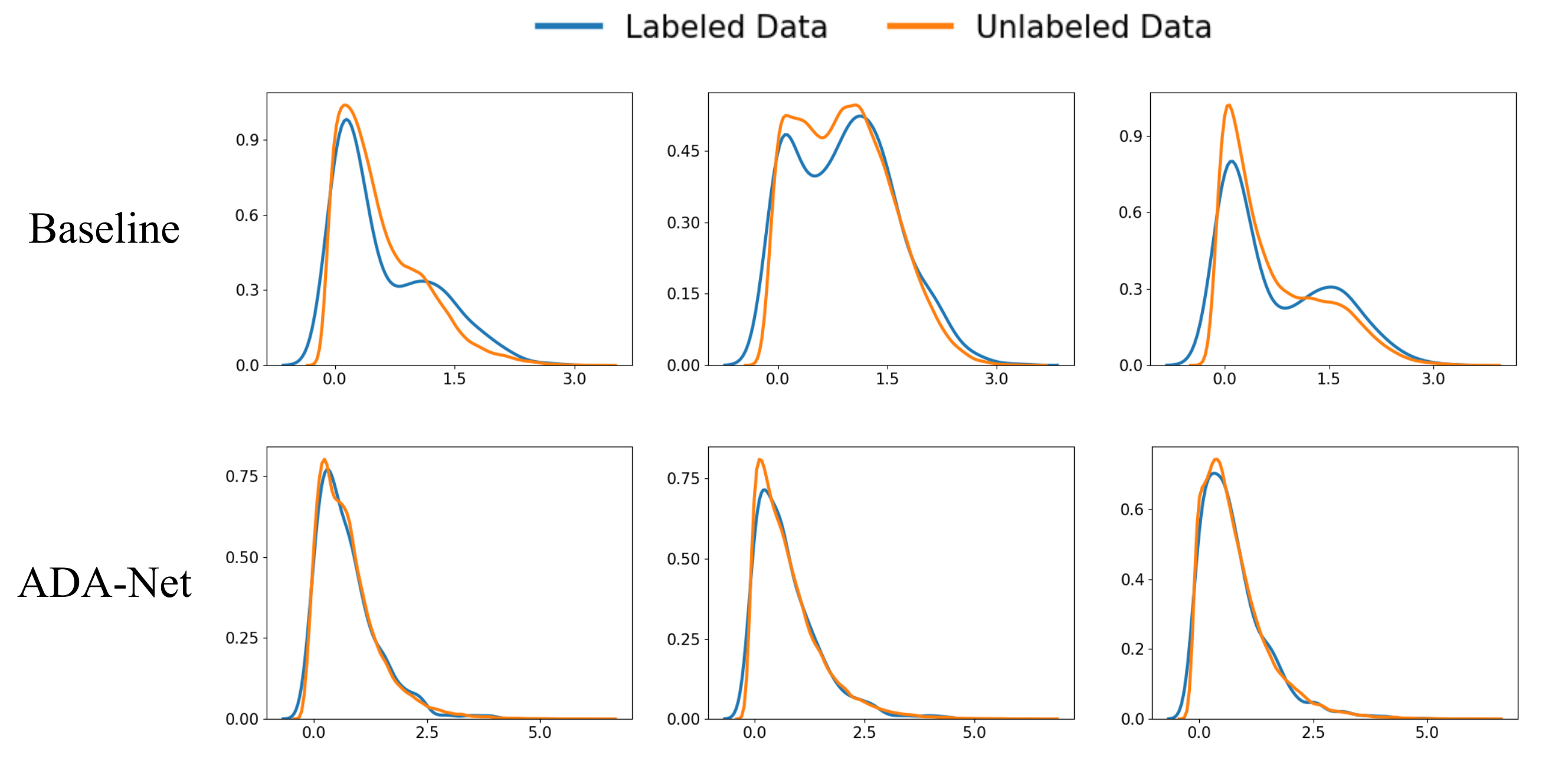}
	\caption{Kernel density estimation results of the baseline Conv-Large method and our ADA-Net for the labeled and unlabeled samples from the SVHN dataset based on the first three feature activations. Considerable data distribution mismatch between labeled and unlabeled data can be observed for the baseline Conv-Large method (top row), while two distributions are generally well aligned after using our ADA-Net (bottom row).}
	\label{fig:damaind}
\vspace{-5pt}
\end{figure}

\paragraf{Results when varying the number of labeled samples:} As discussed in Section~\ref{sec:empd}, in SSL, we often suffer from the data distribution mismatch issue, and the issue becomes more serious when there are fewer labeled samples. To further validate the effectiveness of our ADA-Net, we take the SVHN dataset as an example to conduct the experiments by varying the number of labeled samples. In particular, we train the models by using $200$, $400$, $600$, $800$ and $1,000$ labeled samples, and all other experimental settings remain the same. The error rates of our ADA-Net and the baseline Conv-Large method are plotted in Figure~\ref{fig:size}. We observe that the error rate of the baseline Conv-Large method increases dramatically when reducing the number of labeled samples, which indicates that the SSL problem becomes more challenging due to the sampling bias issue. Nevertheless, our ADA-Net consistently improves the classification performance by reducing such sampling bias with the augmented distribution alignment strategy. The relative improvement is more obvious when there are fewer labeled samples.

\subsubsection{Comparison with the State-of-the-art Methods}
\begin{table}[h!]
	\begin{center}
	    \caption{Classification error rates (\%) of different methods on CIFAR10 and SVNH. The results of the baseline methods are directly quoted from their works. The best results are highlighted in boldface.}
		\label{tab:ssl}
		
\begin{threeparttable}
	\begin{tabular}{c|c|c} 
		\hlineB{3}
		\textbf{Methods} & \textbf{CIFAR10} & \textbf{SVHN} \\
		\hline
		$\Pi$ Model~\cite{laine2017temporal}        & 12.36 & 4.82\\
		Temporal ensembling~\cite{laine2017temporal}& 12.16 & 4.42\\
		VAT~\cite{miyato2018virtual}                & 11.36 & 5.42\\
        VAT+Ent~\cite{miyato2018virtual}            & 10.55 & 3.86\\			
		SaaS~\cite{cicek2018saas}                   & 13.22 & 4.77\\
		MA-DNN~\cite{Chen_2018_ECCV}                & 11.91 & 4.21\\
        VAT+Ent+SNTG~\cite{luo2018smooth}           & 9.89  & 3.83\\
        Mean Teacher+fastSWA~\cite{athiwaratkun2018there}& 9.05 & -\\
        ICT~\cite{verma2019interpolation}           &7.29   &3.89\\
        UPS~\cite{rizve2021in}                      &6.36   &-\\
        \hline
    	ADA-Net (Conv-Large)                              &10.30  &4.62\\
		ADA-Net (VAT+Ent)                    &10.09  &3.54\\
		%ADA-Net\tnote{*} (Ours)                     &8.72   & - \\
		ADA-Net (ICT)                        &\textbf{6.04}   &\textbf{3.49}\\
		\hline 
	\end{tabular}
\end{threeparttable}
    \end{center}
    
\end{table}

We further compare our ADA-Net with the recent state-of-the-art SSL learning approaches~\cite{laine2017temporal,tarvainen2017mean,miyato2018virtual,miyato2018virtual,cicek2018saas,Chen_2018_ECCV,luo2018smooth, athiwaratkun2018there,verma2019interpolation,rizve2021in}. As discussed in \cite{odena2018realistic}, minor modifications of the network structure and the data processing method often lead to different results. For fair comparison, we take the VAT method~\cite{miyato2018virtual} as the reference approach, and strictly follow its experimental setup. In particular, we implement our ADA-Net based on the released code from~\cite{miyato2018virtual}. The same Conv-Large architecture and hyper-parameters are used. For better presentation, we refer to our ADA-Net based on the supervised Conv-Large method as \emph{ADA-Net (Conv-Large)}.
As discussed in Section~\ref{sec:sleda_ext}, we address the SSL problem from the novel perspective of empirical distribution alignment, so our augmented distribution alignment strategy is generally complementary to other SSL methods. Therefore, we also report the results of our ADA-Net (VAT+Ent) and ADA-Net (ICT), by additionally incorporating the two strategies in our ADA-Net into VAT+Ent~\cite{miyato2018virtual} and ICT~\cite{verma2019interpolation}, respectively.

In Table~\ref{tab:ssl}, we report the results of different methods on the CIFAR10 and SVHN datasets. Our ADA-Net method achieves the competitive results when compared with the state-of-the-art SSL methods. Specifically, our ADA-Net (Conv-Large) approach achieves the error rates of $10.30\%$ and $4.62\%$ on CIFAR10 and SVHN, respectively. In addition, our ADA-Net (VAT+Ent) reduces the classification error rate of VAT+Ent from 10.55\% to 10.09\% (\resp from 3.86\% to 3.54\%) on the CIFAR10 (\resp SVHN) dataset. Similarly, our ADA-Net (ICT) reduces the classification error rate of ICT from 7.29\% to 6.04\% (\resp from 3.89\% to 3.49\%) on the CIFAR10 (\resp SVHN) dataset. It is worth mentioning that our ADA-Net (ICT) achieves the new state-of-the-art error rates of 6.04\% and 3.49\% by using the Conv-Large network as the backbone on CIFAR10 and SVHN, respectively. These results clearly demonstrate the effectiveness of our SLEDA framework.

\subsection{Semi-Supervised 3D Point Cloud Recognition}
We further evaluate our proposed approach for the semi-supervised 3D point cloud recognition task.

Point cloud recognition~\cite{wu20153dshapenets,qi2017pointnet,qi2017pointnet2,riegler2017octnet,wang2019dgcnn,rao2020global} has become a popular research topic in recent years, largely due to the rapid adoption of 3D scanning devices in various applications (\eg augmented reality and autonomous driving). Yet, it is more expensive to annotate 3D point clouds than 2D images, since the annotator needs to rotate a 3D point cloud object several times before assigning the label. Therefore, the semi-supervised learning methods provide an effective way to utilize unlabeled point cloud data to improve point cloud recognition performance. Unfortunately, the existing semi-supervised learning methods are not specifically designed for the 3D point cloud recognition task.

In this section, under the SSL setting, we first set up a benchmark for 3D point cloud recognition based on the popular benchmark point cloud datasets ModelNet40 and ShapeNet55. Then we evaluate the state-of-the-art semi-supervised learning methods that were originally designed for image recognition, as well as our ADA-Net, for the point cloud recognition task.

\subsubsection{Datasets}
\paragraf{ModelNet40:} The ModelNet40 dataset \cite{wu20153dshapenets} contains the CAD models from 40 classes. It is split into two subsets, a training set with 9,843 samples and a testing set with 2,468 samples. We follow the previous works \cite{qi2017pointnet,qi2017pointnet2,wang2019dgcnn} to generate the point clouds by sampling points evenly from  the CAD models. In particular, $1,024$ points are sampled for each CAD model in our experiments. For semi-supervised learning, 500 training samples with their labels are used as the labeled training data, while the other training samples are used as unlabeled training data. 

\paragraf{ShapeNet55:} The ShapeNet55 dataset \cite{chang2015shapenet} contains the CAD models from 55 classes. We use the newest version of this dataset (\ie ShapeNetCore.v2), which consists of 35,708 samples in the training set, 5,158 samples in the validation set, and 10,261 samples in the testing set. In our experiments, we combine the training and validation sets as our training set. Similar to the ModelNet40 dataset, we generate the point clouds by sampling $1,024$ points evenly from the CAD models. In our training set, 400 samples are used as labeled data, the rest of the samples are used as unlabeled data.

On both datasets, we perform random sampling to select labeled samples for 10 rounds, and report the average results as well as their standard deviations.

\subsubsection{Implementation Details}
For the semi-supervised point could classification task, we implement our ADA-Net in a similar way as the semi-supervised image classification task. We use PointNet \cite{qi2017pointnet} as the backbone network to process the input point clouds. For each fully-connected layer of the PointNet network, the number of channels is set as 1/4 of the original number of channels in order to reduce the model complexity (\eg, for the first fully-connected layer of our PointNet, we use 16 channels instead of the original 64 channels). Similarly as in the semi-supervised image classification task, a single fully-connected layer is used as the class predictor to map the features to logits, and two fully-connected layers are used as the domain discriminator. 

We set the batch size to 32. The learning rate is initialized as 0.002, and is divided by 10 when 2/3 of the total number of epochs are reached. We train the network for 300 epochs and 80 epochs on ModelNet40 and ShapeNet55, respectively. We use the Adam optimizer with the momentum set as 0.9. In our experiments, we use the following hyper-parameters, \(\text{weight-decay}=0.0001\), interpolation coefficient $\alpha=0.2$ and $0.5$ for ModelNet40 and ShapeNet55, respectively. The experiments on both datasets share the same network and training protocol.

\subsubsection{Experimental Results}
In Table~\ref{tab:ssl-ablation-3d}, the classification error rates of our method ADA-Net and its two variants on the ModelNet40 and ShapeNet55 datasets are reported, in which PointNet~\cite{qi2017pointnet} is used as the backbone network in all methods. On the ModelNet40 and ShapeNet55 datasets, the distribution alignment strategy on the labeled and unlabeled datasets reduces the classification error rates by 1.48\% and 2.05\%, respectively, while the cross-set sample augmentation strategy reduces the classification error rates by 0.51\% and 0.95\%, respectively. By integrating both strategies, our ADA-Net reduces the classification error rate from 24.75\% and 29.39\% (the supervised baseline method PointNet~\cite{qi2017pointnet}) to 20.68\% and 25.92\% (ADA-Net) on the ModelNet40 and ShapeNet55 datasets, respectively. The experimental results clearly demonstrate the effectiveness of our ADA-Net method and the contributions of the two newly proposed strategies for semi-supervised 3D point cloud recognition.

\begin{table}[h!]
	\begin{center}
		\caption{The classification error rates (\%) of our proposed ADA-Net and its variants as well as the supervised baseline method (\ie PointNet~\cite{qi2017pointnet}) on the ModelNet40 and ShapeNet55 datasets. For each method, the error rates are averaged over 10 labeled/unlabeled dataset partitions. ``DAS" denotes the distribution alignment strategy, and ``SAS" denotes the cross-set sample augmentation strategy. PointNet~\cite{qi2017pointnet} is used as the backbone network in all methods. The best average classification results are highlighted in boldface.}
		\label{tab:ssl-ablation-3d}
		\begin{tabular}{c|c c|c|c} 
		\hlineB{3}
			 & {DAS} & {SAS} & {ModelNet40} & {ShapeNet55}\\
			\hline \hline 
			Supervised baseline & &                         &24.75 ($\pm$0.79) &29.39 ($\pm$1.06)\\
			\hline
			\multirow{3}{*}{Ours}   &\checkmark &           &23.27 ($\pm$0.89) &27.89 ($\pm$0.76)\\
			\cline{2-5}
			                        &           &\checkmark &24.24 ($\pm$0.82) &28.44 ($\pm$0.93)\\
    		\cline{2-5}
			                        &\checkmark &\checkmark &\textbf{20.68 ($\pm$0.63)}   &\textbf{25.92 ($\pm$0.91)}\\
			\hline
		\end{tabular}
	\end{center}
\end{table}

\subsubsection{Comparison with Other Methods}
We further compare our ADA-Net with other methods \cite{laine2017temporal,miyato2018virtual,tarvainen2017mean,cicek2018saas,Chen_2018_ECCV,luo2018smooth,athiwaratkun2018there,verma2019interpolation}, which are the state-of-the-art methods for semi-supervised image recognition. Since these methods were designed for the 2D image recognition task and only the image recognition results are reported, we implement all of them for the point cloud recognition task. In our implementation, we utilize the officially released code from all these works. For the backbone network, we use PointNet \cite{qi2017pointnet} with 1/4 of the original number of channels at each layer for each method. We refer to our ADA-Net using the supervised PointNet method as \emph{ADA-Net (PointNet)}.

\begin{table}[h!]
	\begin{center}
		\caption{The classification error rates (\%) of different methods on ModelNet40 and ShapeNet55. For each method, the error rates are averaged over 10 labeled/unlabeled dataset partitions. PointNet~\cite{qi2017pointnet} is used as the backbone network in all methods. The results of all methods are produced based on our implementation. The best average classification results are highlighted in boldface.}
		\label{tab:ssl-3d}
		
\begin{threeparttable}
	\begin{tabular}{c|c|c} 
		    \hlineB{3}
			\textbf{Methods} & \textbf{ModelNet40} & \textbf{ShapeNet55}\\
		    \hline
			$\Pi$ Model~\cite{laine2017temporal}                &22.88 ($\pm$0.78) &27.23 ($\pm$1.23)\\
			Temporal ensembling~\cite{laine2017temporal}        &21.70 ($\pm$0.76) &27.10 ($\pm$0.86)\\
			Mean Teacher~\cite{tarvainen2017mean}               &22.03 ($\pm$0.63) &26.96 ($\pm$0.69)\\
			SaaS~\cite{cicek2018saas}                           &23.03 ($\pm$0.73) &26.63 ($\pm$0.65)\\
			MA-DNN~\cite{Chen_2018_ECCV}                        &22.94 ($\pm$1.15) &28.07 ($\pm$0.97)\\
			VAT~\cite{miyato2018virtual}                        &22.95 ($\pm$1.02) &27.89 ($\pm$1.06)\\
			VAT+Ent~\cite{miyato2018virtual}                    &22.31 ($\pm$1.01) &27.00 ($\pm$1.06)\\
			VAT+Ent+SNTG~\cite{luo2018smooth}                   &22.37 ($\pm$1.10) &26.84 ($\pm$1.04)\\
			Mean Teacher+fastSWA~\cite{athiwaratkun2018there}   &21.75 ($\pm$1.11) &26.57 ($\pm$0.63)\\
			ICT~\cite{verma2019interpolation}                   &21.47 ($\pm$0.95) &26.43 ($\pm$1.06)\\
    		\hline
    		ADA-Net (PointNet)                              &20.68 ($\pm$0.63) &25.92 ($\pm$0.91)\\
    		ADA-Net (VAT+Ent)                               &20.17 ($\pm$0.76) &25.59 ($\pm$0.98)\\
    		ADA-Net (ICT)                                   &\textbf{20.12 ($\pm$0.78)} &\textbf{25.49 ($\pm$0.93)}\\
			\hline
	\end{tabular}
\end{threeparttable}
    %\vspace{-15pt}
    \end{center}
    
\end{table}

In Table \ref{tab:ssl-3d}, we compare the results of our ADA-Net with the baseline methods on ModelNet40 and ShapeNet55. Our ADA-Net achieves the error rates of 20.68\% and 25.92\% on ModelNet40 and ShapeNet55, respectively., which are better than the baseline methods. Moreover, we also demonstrate that for the semi-supervised point cloud recognition task, our ADA-Net method is also complimentary to the existing SSL methods. As shown in Table \ref{tab:ssl-3d}, with our ADA-Net (ICT), we achieve the error rates of 20.12\% and 25.49\% on ModelNet40 and ShapeNet55, respectively, which outperform the state-of-the-art methods for the semi-supervised point cloud recognition task.

To further validate the superiority of our ADA-Net over the existing semi-supervised learning methods, we conduct a significance test based on the experimental results. Specifically, we respectively perform the significance test to compare our ADA-Net (PointNet), ADA-Net (VAT+Ent) and ADA-Net (ICT) with each baseline method based on the results from 10 rounds of experiments with different random labeled/unlabeled dataset partitions. The results demonstrate that on both ModelNet40 and ShapeNet55 datasets, our ADA-Net (PointNet), ADA-Net (VAT+Ent) and ADA-Net (ICT) are significantly better than all the state-of-the-art methods as well as the supervised learning baseline method, as judged by t-test with a significance level of 0.03.

\section{Conclusions}
In this work, we have proposed a new deep semi-supervised learning framework called SLEDA. In particular, we tackle the semi-supervised learning problem from a new perspective, where labeled and unlabeled data often exhibit a considerable empirical distribution mismatch. We prove that the generalization error of semi-supervised learning methods can be bounded by minimizing the training error on labeled data and the empirical distribution distance between labeled and unlabeled data. Based on our framework, we have proposed a new semi-supervised learning method called Augmented Distribution Alignment Network (ADA-Net) by jointly using an adversarial training strategy and a cross-set sample interpolation strategy. These two strategies can be readily incorporated into deep neural networks to boost the existing supervised and semi-supervised learning methods. Comprehensive experiments on two benchmark image datasets CIFAR10 and SVHN and another two benchmark point cloud datasets ModelNet40 and ShapeNet55 have validated the effectiveness of our framework.

\bibliographystyle{IEEEtran}
\bibliography{egbib}

\vspace{-30pt}
\begin{IEEEbiography}[{\includegraphics[width=1in,height=1.25in,clip,keepaspectratio]{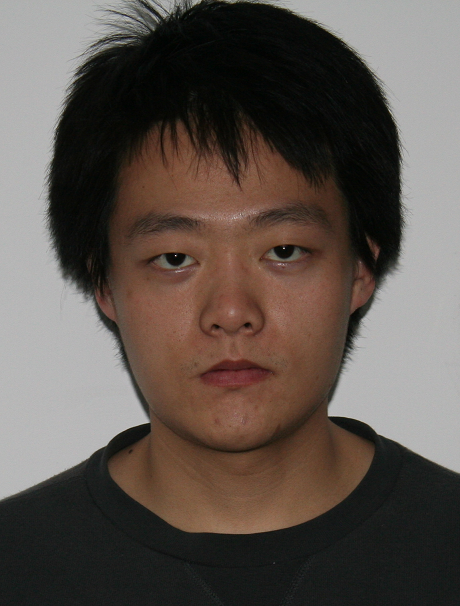}}]{Feiyu Wang}
received the M.Sc. degree in Computer and Systems Engineering from Rensselaer Polytechnic Institute, Troy, NY, USA, in 2017. He is currently a doctoral student at the School of Electrical and Information Engineering, University of Sydney. His current research interests include point cloud recognition, domain adaptation, and semi-supervised learning.
\end{IEEEbiography}

\vspace{-30pt}
\begin{IEEEbiography}[{\includegraphics[width=1in,height=1.25in,clip,keepaspectratio]{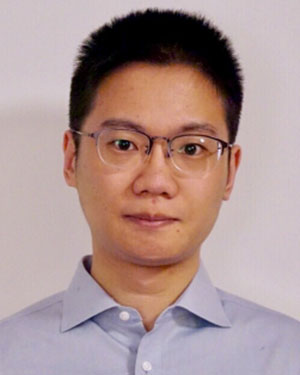}}]{Qin Wang}
received the bachelor‘s degree in electrical engineering from the Department of Electrical Engineering, Tsinghua University, Beijing, China, in 2015 and the master’s degree in energy science and technology in 2018 from the Department of Information Technology and Electrical Engineering, ETH Zürich, Zürich, Switzerland, where he is currently working toward the Ph.D. degree in the field of domain adaptation. His research interests include weakly supervised learning and domain adaptation.
\end{IEEEbiography}

\vspace{-30pt}
\begin{IEEEbiography}[{\includegraphics[width=1in,height=1.25in,clip,keepaspectratio]{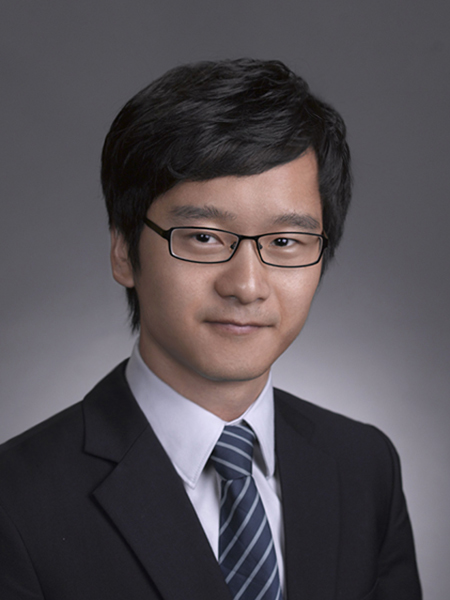}}]{Wen Li}
received the Ph.D. degree from Nanyang Technological University, Singapore, in 2015. From 2015 to 2019, he was a Post-Doctoral Researcher with the Computer Vision Laboratory, ETH Z\"urich, Switzerland. He is currently a Professor with the School of Computer Science and Engineering, University of Electronic Science and Technology of China. His main interests include transfer learning, multi-view learning, multiple kernel learning, and their applications in computer vision. 
\end{IEEEbiography}

\vspace{-30pt}
\begin{IEEEbiography}[{\includegraphics[width=1in,height=1.25in,clip,keepaspectratio]{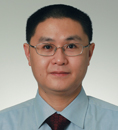}}]{Dong Xu}
received the BE and PhD degrees from University of Science and Technology of China, in 2001 and 2005, respectively. While pursuing the PhD degree, he was an intern with Microsoft Research Asia, Beijing, China, and a research assistant with the Chinese University of Hong Kong, Shatin, Hong Kong, for more than two years. He was a post-doctoral research scientist with Columbia University, New York, NY, for one year. He worked as a faculty member with Nanyang Technological University, Singapore. Currently, he is a professor and chair in Computer Engineering with the School of Electrical and Information Engineering, the University of Sydney, Australia. His current research interests include computer vision, statistical learning, and multimedia content analysis. He was the co-author of a paper that won the Best Student Paper award in the IEEE Conference on Computer Vision and Pattern Recognition (CVPR) in 2010, and a paper that won the Prize Paper award in IEEE Transactions on Multimedia (T-MM) in 2014. He is a fellow of the IEEE.
\end{IEEEbiography}

\vspace{-30pt}
\begin{IEEEbiography}[{\includegraphics[width=1in,height=1.25in,clip,keepaspectratio]{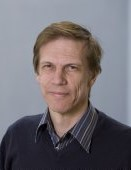}}]{Luc Van Gool}
received the degree in electromechanical engineering at the Katholieke Universiteit Leuven, in 1981. Currently, he is a professor at the Katholieke Universiteit Leuven in Belgium and the ETH in Z\"urich, Switzerland. He leads computer vision research at both places, where he also teaches computer vision. He has authored more than 200 papers in this field, and  received several Best Paper awards. He has been a program committee member of several major computer vision conferences. His main interests include 3D reconstruction and modeling, object recognition, tracking, and gesture analysis. He is a co-founder of 10 spin-off companies.
\end{IEEEbiography}
\hfill

\end{document}